\def\NAMEDBUILD{}
\documentclass{article}
\usepackage{iclr2027_conference,times}
\usepackage{amsmath,amsfonts,bm}

\def\eqref#1{equation~\ref{#1}}
\def\1{\bm{1}}

\DeclareMathAlphabet{\mathsfit}{\encodingdefault}{\sfdefault}{m}{sl}
\SetMathAlphabet{\mathsfit}{bold}{\encodingdefault}{\sfdefault}{bx}{n}

\newcommand{\softmax}{\mathrm{softmax}}

\usepackage[hypertexnames=false]{hyperref}
\usepackage{url}
\usepackage{booktabs}
\usepackage{array}
\usepackage{float}
\usepackage{tikz}
\usepackage{amsmath,amssymb,amsthm}
\newcommand{\vkEightKxEight}{1.00}  % digestwidth_8192.json dig_r64 rho=1/8
\newcommand{\fullThirtyTwoKxThirtyTwo}{0.92}  % sidecar_32768.json sel_k64
\newcommand{\fullThirtyTwoKxOTE}{0.83}  % sidecar_32768.json sel_k64
\newcommand{\vkSixtyFiveKxThirtyTwo}{0.92}  % sidecar_65536.json digk64
\newcommand{\vkSixtyFiveKxOTE}{0.58}  % sidecar_65536.json digk64
\newcommand{\fullSixtyFiveKxThirtyTwo}{0.92}  % sidecar_65536.json sel_k64
\newcommand{\fullSixtyFiveKxOTE}{0.58}  % sidecar_65536.json sel_k64
\newcommand{\ropeVKxThirtyTwo}{0.08}  % e5_dsv2.json kcover_dig rho=1/32
\newcommand{\ropeEvictxThirtyTwo}{0.42}  % e5_dsv2.json sel_k16 rho=1/32
\newcommand{\rowNormMaxKimi}{3.48}  % dissect_kimi.json max rot/content row norm
\newcommand{\rowNormMaxDsv}{1.01}  % dissect_dsv2.json
\newcommand{\blHtwoOxEight}{0.00}  % baselines_8192.json h2o
\newcommand{\blHtwoOxThirtyTwo}{0.00}  % baselines_8192.json h2o
\newcommand{\blHtwoOxOTE}{0.00}  % baselines_8192.json h2o
\newcommand{\blSnapKVxEight}{0.33}  % baselines_8192.json snapkv
\newcommand{\blSnapKVxThirtyTwo}{0.04}  % baselines_8192.json snapkv
\newcommand{\blSnapKVxOTE}{0.00}  % baselines_8192.json snapkv
\newcommand{\blHRecxEight}{0.00}  % h2orecent_8192.json h2o_recent
\newcommand{\blHRecxThirtyTwo}{0.00}  % h2orecent_8192.json h2o_recent
\newcommand{\blHRecxOTE}{0.00}  % h2orecent_8192.json h2o_recent
\newcommand{\blHtwoOLxLThirtyTwo}{0.00}  % baselines_32768.json h2o
\newcommand{\blHtwoOLxLOTE}{0.00}  % baselines_32768.json h2o
\newcommand{\blSnapKVLxLThirtyTwo}{0.00}  % baselines_32768.json snapkv
\newcommand{\blSnapKVLxLOTE}{0.00}  % baselines_32768.json snapkv
\newcommand{\ablFourxThirtyTwo}{0.42}  % digestwidth_8192.json dig_r4
\newcommand{\ablFourxOTE}{0.12}  % digestwidth_8192.json dig_r4
\newcommand{\ablEightxThirtyTwo}{0.50}  % digestwidth_8192.json dig_r8
\newcommand{\ablEightxOTE}{0.17}  % digestwidth_8192.json dig_r8
\newcommand{\ablSixteenxThirtyTwo}{0.54}  % digestwidth_8192.json dig_r16
\newcommand{\ablSixteenxOTE}{0.38}  % digestwidth_8192.json dig_r16
\newcommand{\ablThirtyTwoxThirtyTwo}{0.75}  % digestwidth_8192.json dig_r32
\newcommand{\ablThirtyTwoxOTE}{0.62}  % digestwidth_8192.json dig_r32
\newcommand{\ablSixtyFourxThirtyTwo}{0.88}  % digestwidth_8192.json dig_r64
\newcommand{\ablSixtyFourxOTE}{0.67}  % digestwidth_8192.json dig_r64
\newcommand{\scoreVarShare}{67.6\%}  % branch.py output, session log 2026-09-01: rope branch share of score std
\newcommand{\branchTopOne}{0.9869}  % branch.py log: rope-gauge top-1 retention, 1/128 column (comment previously mislabeled 1/32)
\newcommand{\contentTopOne}{0.0001}  % branch.py log: nope-gauge top-1 retention, 1/128 column
\newcommand{\readCost}{11\%}  % 64/576
\newcommand{\ttxThirtyTwo}{1.00}  % twotier_8192.json twotier
\newcommand{\ttxOTE}{1.00}  % twotier_8192.json twotier
\newcommand{\ttMaskExtra}{0.5--2.6\%}  % tt instrumented run: per-query extra archived rows admitted
\newcommand{\trigSoundRate}{99.5\%}  % cert_trigger.log: sound Cauchy-Schwarz bound trigger rate
\newcommand{\trigConfRec}{0.941}  % gauss_trigger.log
\newcommand{\trigGaussSigma}{3.5--10.9$\sigma$}  % gauss_trigger.log: realized argmax residual quantiles
\newcommand{\rsSixtyFourFire}{0.028}  % ranksweep.log
\newcommand{\rsSixtyFourRows}{15.2}  % ranksweep.log
\newcommand{\rsSixtyFourRec}{0.950}  % ranksweep.log
\newcommand{\rsOneTwoEightFire}{0.008}  % ranksweep.log
\newcommand{\rsOneTwoEightRows}{3.6}  % ranksweep.log
\newcommand{\rsOneTwoEightRec}{0.931}  % ranksweep.log
\newcommand{\rsOneNineTwoFire}{0.004}  % ranksweep.log
\newcommand{\rsOneNineTwoRows}{1.2}  % ranksweep.log
\newcommand{\rsOneNineTwoRec}{0.974}  % ranksweep.log
\newcommand{\rsTwoFiveSixFire}{0.002}  % ranksweep.log
\newcommand{\rsTwoFiveSixRows}{1.2}  % ranksweep.log
\newcommand{\rsTwoFiveSixRec}{0.962}  % ranksweep.log
\newcommand{\pdKimiWorst}{0.67}  % pagedig.log: Kimi rec@8 r_mean, worst layer (11)
\newcommand{\pdDsvWorst}{0.024}  % pagedig.log: DSV2 rec@8 r_mean, worst layer (19)
\newcommand{\pdKimiSoundWorst}{0.97}  % pagedig.log: Kimi rec@8 r_max (sound), worst layer
\newcommand{\pdDsvSoundWorst}{0.014}  % pagedig.log: DSV2 rec@8 r_max, worst layer
\newcommand{\spreadInfl}{1.78--1.86$\times$}  % clustererr session records 2026-09-01: within-cluster spread, rotated vs not
\newcommand{\unionKimi}{2.3--6.7\%}  % entropy.py union sweep, session log 2026-09-02: top-1 union, Kimi
\newcommand{\unionDsv}{10.2--46.8\%}  % entropy.py union sweep: top-1 union, DSV2
\newcommand{\nVerdicts}{20}  % PREREG1-20
\newcommand{\nClosedRoutes}{8}  % DELIVERABLE.md section 3
\newcommand{\nProxyInversions}{10}  % PREREG9 verdict
\newcommand{\bpbTTxEight}{+0.0011}  % bpb_8192.json twotier ce
\newcommand{\bpbTTxThirtyTwo}{+0.0018}  % bpb_8192.json twotier ce
\newcommand{\bpbTTxOTE}{+0.0022}  % bpb_8192.json twotier ce
\newcommand{\bpbDigxEight}{+0.0055}  % bpb_8192.json dig_r64 ce
\newcommand{\bpbDigxThirtyTwo}{+0.0090}  % bpb_8192.json dig_r64 ce
\newcommand{\bpbDigxOTE}{+0.0104}  % bpb_8192.json dig_r64 ce
\newcommand{\bpbRecxEight}{+0.0055}  % bpb_8192.json recent ce
\newcommand{\bpbRecxThirtyTwo}{+0.0093}  % bpb_8192.json recent ce
\newcommand{\bpbRecxOTE}{+0.0112}  % bpb_8192.json recent ce
\newcommand{\vkEightKxThirtyTwo}{0.89}  % pooled 32/36 seeds 11+12
\newcommand{\vkEightKxSixtyFour}{0.79}  % pooled 19/24 seeds 11+12
\newcommand{\vkEightKxOTE}{0.69}  % pooled 25/36 seeds 11+12
\newcommand{\fullEightKxThirtyTwo}{0.86}  % pooled 31/36
\newcommand{\fullEightKxSixtyFour}{0.79}  % pooled 19/24
\newcommand{\fullEightKxOTE}{0.69}  % pooled 25/36
\newcommand{\vkThirtyTwoKxThirtyTwo}{0.96}  % pooled 23/24
\newcommand{\vkThirtyTwoKxSixtyFour}{0.92}  % pooled 22/24
\newcommand{\vkThirtyTwoKxOTE}{0.83}  % pooled 20/24
\newcommand{\fullThirtyTwoKxSixtyFour}{0.83}  % gapfix64_32768.json
\newcommand{\fullSixtyFiveKxSixtyFour}{0.75}  % gapfix64_65536.json
\newcommand{\vkThirtyTwoKxEight}{1.00}  % gap_32768.json
\newcommand{\blHtwoOLxLEight}{0.00}  % gap_32768.json
\newcommand{\blSnapKVLxLEight}{0.42}  % gap_32768.json
\newcommand{\wciEightKxThirtyTwo}{32/36\,[0.75,0.96]}  % pooled 8k
\newcommand{\wciThirtyTwoKxThirtyTwo}{23/24\,[0.80,0.99]}  % pooled 32k
\newcommand{\wciEightKxSixtyFour}{19/24\,[0.60,0.91]}  % pooled 8k
\newcommand{\wciThirtyTwoKxSixtyFour}{22/24\,[0.74,0.98]}  % pooled 32k
\newcommand{\wciEightKxOTE}{25/36\,[0.53,0.82]}  % pooled 8k
\newcommand{\wciThirtyTwoKxOTE}{20/24\,[0.64,0.93]}  % pooled 32k
\newcommand{\wciSixtyFiveKxThirtyTwo}{11/12\,[0.65,0.99]}  % sidecar_65536.json
\newcommand{\wciSixtyFiveKxSixtyFour}{9/12\,[0.47,0.91]}  % floor64_65536.json
\newcommand{\wciSixtyFiveKxOTE}{7/12\,[0.32,0.81]}  % sidecar_65536.json
\newcommand{\wciMFour}{11/12\,[0.65,0.99]}  % m4_needle.json
\newcommand{\wciMFiveTier}{12/12\,[0.76,1.00]}  % m5_needle.json
\newcommand{\wciMFiveTK}{12/12\,[0.76,1.00]}  % m5_32k.json
\newcommand{\wciTT}{24/24\,[0.86,1.00]}  % twotier_8192.json 1/128
\newcommand{\mFiveTKTier}{12/12}  % m5_32k.json
\newcommand{\vkSixtyFiveKxSixtyFour}{0.75}  % floor64_65536.json
\newcommand{\bpbEng}{0.5879}  % m6_bpb.json bpb_eng=0.587922
\newcommand{\bpbVk}{0.5889}  % m6_bpb.json bpb_vk=0.588913
\newcommand{\dceVkMed}{+0.0020}  % m6_bpb.json dce_vk_median=0.001971 nats/token
\newcommand{\readBytesRatio}{$\sim$26\%}  % 294/1152 B per row, 16-bit index
\newcommand{\serveCrossover}{$\sim$40k}  % PR-tree fe4430b2 rerun, server-side, both arms single-launch
\newcommand{\serveSpeedTwoFiveSix}{1.18$\times$}  % PR-tree rerun, 256k bucket (272k: 1.17x, 496k: 1.39x)
\newcommand{\vkTaxUs}{68\,\textmu s}  % Gate 24 nsys, per-step fused VestigeKV adder
\newcommand{\gsmBase}{0.821}  % Gate 25b, 64-shot n=800, same-tree triton
\newcommand{\gsmVk}{0.830}  % PR 09c39628 rerun, n=800 (vk +0.9pp, well inside noise)
\newcommand{\mauveEngFused}{0.930}  % serving path, results/quality_mauve_scores.json
\newcommand{\mauveVkFused}{0.999}  % PR 09c39628 rerun, serving path
\newcommand{\serveSpeedMax}{1.39$\times$}  % final curve, 496k bucket
\newcommand{\tputGainTwelve}{8.0\%}  % PR-tree rerun bs=12: 207.2 vs 191.9 tok/s

\newtheorem{proposition}{Proposition}
\newtheorem{lemma}{Lemma}
\newtheorem{corollary}{Corollary}
\newcommand{\dividend}[2]{\par\smallskip\noindent\textit{NoPE dividend: #1
Under RoPE: #2}}
\ifdefined\NAMEDBUILD\iclrfinalcopy\fi

\title{VestigeKV: The NoPE-MLA KV Cache Carries\\Its Own Sparse-Attention Signal in a Vestigial Branch}
\ifdefined\NAMEDBUILD
  \author{%
    WenJie Fan\\
    Yotta Labs\\
    \texttt{fanwj@mail.ustc.edu.cn} \quad \texttt{fanwenjie@yottalabs.ai}\\
    \texttt{https://github.com/fan-wenjie/vestigekv}
  }
\else
  \author{Anonymous authors\\Paper under double-blind review}
\fi

\begin{document}
\maketitle
% The style stamps "Published as a conference paper" inside \@maketitle when
% final; a preprint is not published, so the named build corrects the head here.
\ifdefined\NAMEDBUILD\lhead{Preprint}\fi

\begin{abstract}
A long-lived KV cache must be compressed \emph{before} the queries that will
read it exist. Selection by observed attention collapses there: on a NoPE-MLA
model, H2O and SnapKV retrieve \blHtwoOxEight{} and \blSnapKVxEight{} of
needles at $8\times$ compression, because a token's importance has not yet
been observed. VestigeKV instead derives a sparse attention pattern from a
signal the cache already carries, occupying the sparse-attention literature's
one unoccupied quadrant: training-free \emph{and} query-independent. In NoPE-MLA the 64-dimensional decoupled branch is a vestige of
RoPE that training repurposes into a salience channel; reading \readCost{}
of each row, it partitions the cache into an attended tier and a
GPU-resident archive that no row ever leaves, reachable each step by a
certified, query-adaptive trigger. Nothing is trained and cache rows are
never quantized, so every quality effect attributes to selection and
scheduling. On Kimi Linear 48B, retrieval holds at \vkEightKxEight{} under
$8\times$ and \vkThirtyTwoKxThirtyTwo{} under $32\times$ from 8k to 65k
context, with zero gap to full-row selection, and the recall tier holds
$128\times$ at \ttxOTE{} (8k). Both tiers stay on the GPU, so the win is
\emph{speed}, not memory: the per-step scan reads \readBytesRatio{} of the
bytes dense attention would, and on a two-node \texttt{sglang} deployment
the crossover sits at \serveCrossover{} context, reaching
\serveSpeedTwoFiveSix{} at 256k and \serveSpeedMax{} at 496k. The mechanism
is exclusive to NoPE: the identical operator on a RoPE MLA collapses to
\ropeVKxThirtyTwo{}, query-independent salience exists only without
rotation, and query-universal exact merging is provably impossible under
RoPE. All thresholds were frozen before their data; \nVerdicts{} archived
verdicts and \nClosedRoutes{} closed routes accompany the paper.
\end{abstract}

\section{Introduction: capability, cost, and required changes}
For a reader holding stock Kimi Linear weights, the offer is as follows (for
Kimi K3's Gated-MLA\footnote{K3's NoPE Gated-MLA architecture is reported in
its public release materials; no archival citation exists at submission time.}
the \emph{machinery} extends by proof ---
Section~\ref{sec:math}d --- while the signal content is untested; see
Limitations):

\begin{table}[H]
\centering\small
\begin{tabular}{l>{\raggedright\arraybackslash}p{8.9cm}}
\toprule
effect & attended tier $576\to18$ dims/token at $32\times$ (archive bit-exact,
 GPU-resident) with retrieval \vkThirtyTwoKxThirtyTwo{}
 (8k--65k); lossless at $8\times$; with the recall tier, \ttxOTE{} at $128\times$\\
selection cost & read \readCost{} of each row once per compression event;
 one low-pass filter ($O(T\log T)$) $+$ top-$m$; decode path unchanged\\
model changes & none --- weights, kernels, arithmetic untouched; this is a cache
 \emph{policy}: read \texttt{latent\_cache[$\ldots$,512:]}, rank, free rows\\
recall tier (standard) & archived rows stay GPU-resident $+$ a $(64{+}r)$-dim index;
 per step, one fused index scan and the certified fired set admitted --- full
 attention reads drop from 1152 to ${\sim}294$ bytes/token with the 16-bit
 index --- $3.9\times$ on the KV path \emph{in bytes} (Eq.~\ref{eq:reads});
 both tiers GPU-resident: the win is speed, not memory\\
NoPE exclusivity & the signal is query-independent salience, which exists only when the
 score is one fixed bilinear form; measured \ropeVKxThirtyTwo{} under RoPE\\
\bottomrule
\end{tabular}
\end{table}

Structurally, VestigeKV is a sparse attention; what distinguishes it is
the \emph{origin of the pattern}. Fixed geometry
\citep{beltagy2020longformer,zaheer2020bigbird,xiao2024streamingllm}
ignores content; trained patterns \citep{yuan2025nsa,lu2025moba} cost
pre-training; score-based selection
\citep{zhang2023h2o,li2024snapkv,tang2024quest} is training-free but
query-dependent. VestigeKV occupies the remaining quadrant ---
\emph{training-free and query-independent} --- because the pretrained NoPE
cache already carries the pattern in its vestigial branch; that quadrant
is exactly what compression-before-the-query requires
(Section~\ref{sec:problem}), and it exists only under NoPE
(Table~\ref{tab:rope}).

The paper owes the reader three things: the mathematics that makes a
query-independent signal possible and merging impossible
(Section~\ref{sec:math}); the engineering, down to a recall tier with
per-step certificates (Section~\ref{sec:eng}); and the measurements
(Section~\ref{sec:exp}). Each piece names the NoPE dividend and the RoPE
obstruction.

\section{The problem: the cache is compressed before its queries exist}
\label{sec:problem}
In a long-lived cache, the query that will need a token arrives after the
compression decision. Any statistic of \emph{observed} attention is blind to that
query: on
our NoPE-MLA testbed, H2O drops to \blHtwoOxEight{} needle retrieval at $8\times$
where VestigeKV holds \vkEightKxEight{} (full comparison in
Table~\ref{tab:baselines}), and giving H2O half its budget as a recent window
changes nothing --- the failure is informational, not budgetary.

(Timeline illustration: Appendix~\ref{app:exp_tables}, Figure~\ref{fig:timeline}.)

\dividend{under a fixed bilinear form, salience is a property of the token alone, so
a signal can survive compression-before-query (Section~\ref{sec:math}a).}{the
winner of attention depends on the query's position; the pre-query setting has no
usable signal, which is why the RoPE-era literature selects on observed attention
and inherits this collapse.}

\section{Theory: the mathematics of a position-free cache}\label{sec:math}
The cache row and the two maps that read it:
\begin{equation}
\tilde C_u=[\hat c_u;\,r_u],\qquad
s_h(t,u)=\lambda\big\langle Q_t^h,\tilde C_u\big\rangle,\qquad
v_u^h=W_{UV}^h\hat c_u,
\label{eq:setup}
\end{equation}
with $Q_t^h=[W_{UK}^{h\top}q_t^{h,c};\,q_t^{h,r}]$ the absorbed query and
$\lambda$ the attention scale. Everything
follows from one structural fact: under NoPE, \emph{both} maps depend on $u$ only
through the row itself.

\begin{lemma}[Exchangeability]\label{lem:exch}
NoPE-MLA attention output is a function of the multiset $\{\!\{\tilde C_u\}\!\}$:
$o_t^h=\sum_u e^{s_h(t,u)}v_u^h\big/\sum_u e^{s_h(t,u)}$, and both sums are
symmetric in $u$ by Eq.~\ref{eq:setup}. \qed
\end{lemma}

Three consequences, one per pillar of the paper.

\begin{sloppypar}
\textbf{(a) Query-independent salience exists (the selector).} The score is one
bilinear form, so $\arg\max_u s_h(t,u)=\arg\max_u\langle A_h^\top x_t,x_u\rangle$,
writing $x$ for hidden states and $A_h$ for the per-head bilinear form with
$\operatorname{rank}A_h\le d_{\mathrm{head}}$; every read direction lies in one
fixed low-rank subspace, so the tokens that can ever win are the extreme points of
the hidden-state cloud against one fixed cone --- measured at \unionKimi{} of tokens.
Under RoPE the form is the family $A_h(\Delta)=W_q^\top R_\Delta W_k$; the cone
sweeps with position and the winner set inflates to \unionDsv{}. A compression
decision made before its queries can only rely on a query-independent signal, and
only NoPE has one.
\end{sloppypar}

\textbf{(b) Merging is impossible; selection is forced (the method).}
\begin{proposition}[Minimum exact cache]\label{prop:impossible}
Over any open set of queries, a representative cache $\{(C_g,b_g)\}_{g=1}^{B}$ with
logit offsets reproducing attention exactly for all queries satisfies
$B\ge\#\{\text{distinct rows}\}$.
\end{proposition}
\begin{proof}[Proof sketch]
Exactness for all $Q$ means
$\sum_g e^{b_g}e^{\langle Q,C_g\rangle}=\sum_u e^{\langle Q,\tilde C_u\rangle}$ (and
the $v$-weighted analogue). The characters $Q\mapsto e^{\langle Q,C\rangle}$ of
$(\mathbb{R}^d,+)$ are linearly independent on any open set, so the representative
multiset must contain every distinct row with weight equal to its multiplicity. Full
proof in Appendix~\ref{app:proofs}.
\end{proof}
Under RoPE equal content at distinct positions is never equal as rows,
$\lVert R_uk-R_vk\rVert^2=\sum_j4\sin^2\tfrac{\theta_j(u-v)}{2}\lVert
k^{(j)}\rVert^2>0$, so $B=T$: \emph{exact} compression is zero, and every practical
method must select. Under NoPE$+$MLA the merge class is wider than equality
(a corollary: $c_v=\alpha c_u,\ \alpha>0$, $r_v=r_u$ suffices, by positive homogeneity
of RMSNorm) --- but a pre-registered measurement finds this class empty on real
corpora, so selection is forced there too.

\textbf{(c) Per-token certificates (the recall tier).} For any surrogate row
$C_{g(u)}$ with $\lVert\tilde C_u-C_{g(u)}\rVert\le\varepsilon_u$ and any future
query, $|\Delta s|\le\lambda\lVert Q\rVert\varepsilon_u$ --- a query-\emph{uniform}
bound, available only because the form is fixed (under RoPE it must hold over the
whole rotation orbit). Section~\ref{sec:tier} builds the archive index from exactly
this decomposition.

\paragraph{(d) Gate commutation (extension to Gated-MLA).} Every guarantee
above concerns the attended \emph{set} --- which rows enter the softmax
multiset --- and none concerns what happens to the aggregate afterwards. So
for any attention variant of the form
$o_t = G\big(x_t,\ \textstyle\sum_{u\in S} p_u(S)\,v_u\big)$
with $p$ the softmax of the same fixed NoPE bilinear score over the
attended set $S$, the machinery transfers verbatim: partition exactness,
exchangeability (Lemma~\ref{lem:exch}), the certificate of (c), and
ranking stationarity all hold for $S$ regardless of $G$, and output
perturbation bounds carry over up to $G$'s Lipschitz constant in its second
argument --- for an elementwise sigmoid output gate that constant is at
most one, so the bound \emph{tightens}. Query-side positive gating of the
score preserves the per-query ranking (a positive factor is monotone), and
a query-independent per-row gate absorbs into the row exactly as the
residual scalar does. The one form that does not transfer is a
query-row-\emph{coupled} score gate, which breaks the fixed bilinear form
--- the same obstruction as rotation, and the boundary of the theory.
Kimi K3 is reported to use a NoPE Gated-MLA of the output-gated form; the
\emph{machinery} therefore extends by the argument above, while the
\emph{signal content} --- whether K3's training reallocates its un-roped
branch into a salience channel as Kimi Linear's does
(Section~\ref{sec:anatomy}) --- is an empirical property of K3's training
that this study does not measure.

\ifdefined\NAMEDBUILD
\subsection{Eviction commutes with time}\label{app:stationary}
A consequence of Lemma~\ref{lem:exch} worth isolating, because it licenses a
schedule the experiments use and RoPE cannot.

\begin{corollary}[Ranking stationarity]\label{cor:stationary}
Under NoPE, $s(q_t,u)=\lambda\langle Q_t,\tilde{C}_u\rangle$
(Eq.~\ref{eq:setup}) contains neither $t$ nor the age
of row $u$. Hence any row-intrinsic ranking (such as $\sigma$), once computed,
is valid at every later step, and a global top-$m$ maintained incrementally
(a heap, later rows displacing earlier ones) equals the one-shot global
top-$m$ computed at the end. Eviction decisions are final.
\end{corollary}

The measured footprint is the streaming result of
Section~\ref{sec:exp}: blockwise compression equals one-shot global
compression trial by trial at both ratios, and so does blockwise compression
with global rebalance (constant $m$, later rows displacing earlier ones) at
matched final budget. Under RoPE the corollary fails at
an intrinsic rate: the cached score
$s(q_t,u)=\lambda\, q_t^{\top}R_{t-u}W_{K}\tilde{C}_u=\sum_i a_i(u)\cos(\omega_i(t-u)+\phi_i(u))$
oscillates per frequency band, so the pairwise order of any two rows flips as
$t$ grows and a frozen ranking goes stale; the only remedy is re-scoring every
archived row against the current position at every rebalance, so no eviction
decision is ever final. NoPE deletes this staleness failure mode entirely;
the remaining failure of a constant-$m$ schedule is genuine budget contention
(a later high-$\sigma$ row displacing a needle), a memory-for-quality trade,
not signal failure. The schedule itself is prior art
(H2O~\citep{zhang2023h2o} lineage); the dividend attaches to the signal.

\fi

\subsection{The vestige, and the new job training gave it}\label{sec:anatomy}
Which rows to select is decided by anatomy. The branch $r_u$ exists because rotary
position cannot pass MLA's low-rank bottleneck ($R_{t-u}$ does not commute with the
up-projection), so DeepSeek-V2 routes position around it in a per-token, MQA-shared
channel. Under NoPE the released code caches the branch but never rotates it (no
rotary call in the reference implementation; \texttt{skip\_rope=True} in
\texttt{sglang}): the positional job is gone. Training reallocates the freed
channel into the cache's salience channel; the measurements behind this account
(branch/content norm ratios, cross-model NoPE vs.\ RoPE; Table~\ref{tab:anatomy})
are in Appendix~\ref{app:exp_tables}.

\dividend{a 64-dim channel with no positional duty becomes, under training, the
cache's magnitude/salience channel --- readable at \readCost{} cost.}{the same 64
dims must encode rotation frequencies; their low-pass residual is phase, not
salience, and the identical statistic scores \ropeVKxThirtyTwo{}
(Table~\ref{tab:rope}).}

\paragraph{The read budget, computed.}\label{par:reads}
Decode-time attention is bandwidth-bound: KV-path time tracks bytes read per
step. The three-term budget follows from the construction above:
\begin{equation}\label{eq:reads}
\mathrm{reads}(\rho,r) = \underbrace{\rho\cdot 576}_{\text{attended}}
+ \underbrace{(1-\rho)(64+r+1)}_{\text{index scan}}
+ \underbrace{f\cdot 576}_{\text{admitted}}\ \text{dims/token},
\end{equation}
where the $+1$ counts the per-row residual scalar $\eta_u$. In \emph{bytes}
(what bandwidth prices): rows are $1152$\,B, the 16-bit index scan reads its
$(64{+}r{+}1)$ dims at 2\,B (${\approx}260$\,B/row), and the budget lands at ${\sim}294$ vs.\
$1152$\,B/token --- $3.9\times$ (a full-precision index blunts this to
$2.1\times$). At $\rho{=}1/32$, $r{=}64$ the admitted fraction $f$
measures $0.4$--$0.8\%$ on document decode ($0.5$--$2.6\%$ on needle
contexts); the trigger term is under $1\%$. The ratio is asymptotic in
context and batch --- attention is per-request-unamortizable (what the
method cuts) while shared weights amortize --- so the reported ratios are
a conservative floor. The realized pipeline is seven fused kernels inside
the decode CUDA graph (Section~\ref{sec:eng}); the per-step adder is
\vkTaxUs{}.

\section{Engineering}\label{sec:eng}

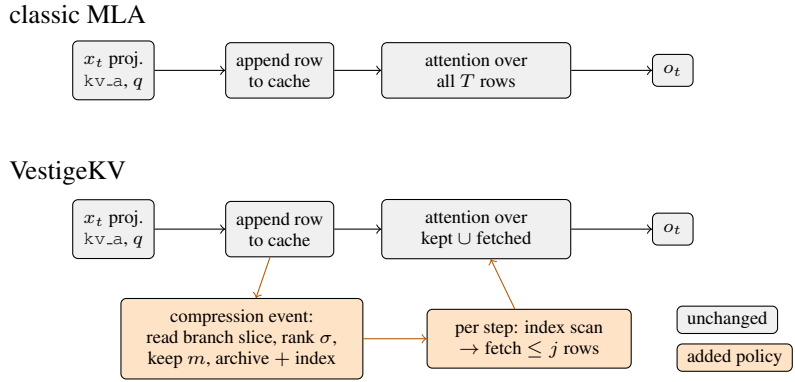
\begin{figure}[H]
\centering
\begin{tikzpicture}[font=\scriptsize,x=1cm,y=1cm,
  same/.style={draw,rounded corners=2pt,align=center,inner sep=4pt,fill=gray!12},
  new/.style={draw,rounded corners=2pt,align=center,inner sep=4pt,fill=orange!22}]
  % lane labels
  \node[anchor=west] at (-0.6,3.05) {\normalsize classic MLA};
  \node[anchor=west] at (-0.6,0.95) {\normalsize VestigeKV};
  % lane A: classic
  \node[same] (a1) at (0.9,2.3) {$x_t$ proj.\\\texttt{kv\_a}, $q$};
  \node[same] (a2) at (3.1,2.3) {append row\\to cache};
  \node[same,minimum width=2.5cm] (a3) at (5.7,2.3) {attention over\\all $T$ rows};
  \node[same] (a4) at (8.3,2.3) {$o_t$};
  \draw[->] (a1)--(a2); \draw[->] (a2)--(a3); \draw[->] (a3)--(a4);
  % lane B: ours
  \node[same] (b1) at (0.9,0.2) {$x_t$ proj.\\\texttt{kv\_a}, $q$};
  \node[same] (b2) at (3.1,0.2) {append row\\to cache};
  \node[same,minimum width=2.5cm] (b3) at (5.7,0.2) {attention over\\kept $\cup$ fetched};
  \node[same] (b4) at (8.3,0.2) {$o_t$};
  \draw[->] (b1)--(b2); \draw[->] (b2)--(b3); \draw[->] (b3)--(b4);
  % delta boxes
  \node[new,minimum width=3.2cm] (comp) at (2.6,-1.25) {compression event:\\read branch slice, rank $\sigma$,\\keep $m$, archive $+$ index};
  \node[new,minimum width=2.7cm] (scan) at (6.4,-1.25) {per step: index scan\\$\to$ fetch $\le j$ rows};
  \draw[->,orange!70!black] (b2)--(comp);
  \draw[->,orange!70!black] (comp)--(scan);
  \draw[->,orange!70!black] (scan)--(b3);
  % legend
  \node[same,anchor=west,inner sep=2pt] at (8.35,-1.0) {\ unchanged\ };
  \node[new,anchor=west,inner sep=2pt] at (8.35,-1.5) {\ added policy\ };
\end{tikzpicture}
\caption{The delta from classic MLA, per layer. Every model operation (gray) is
byte-identical --- projections, kernels, arithmetic, weights. What is added
(orange) is cache-membership policy: a compression event that ranks rows by the
branch slice, and the per-step scan-and-fetch that redeems its no-row-deleted
promise. The two are halves of one algorithm --- without the recall, the
partition is just lossy eviction (Appendix~\ref{app:tier1}). Removing the \emph{whole}
orange path (the kill-switch) recovers classic MLA exactly, byte-for-byte.}
\label{fig:delta}
\end{figure}

\subsection{The eviction policy}\label{sec:method}
For a closed prefix block of branch vectors $r_{0:T}$, $\mathcal{F}$ the
sequence-axis rFFT, bandwidth $\kappa$:
\begin{equation}
\sigma_u=\big\lVert r_u-\big(\mathcal{F}^{-1}\mathbb{1}_{[0,\kappa)}\mathcal{F}
r\big)_u\big\rVert
\quad\text{per fixed 4096-token block};
\end{equation}
the transform runs over fixed-length blocks (the close granularity), never the
whole prefix: $\kappa$ counts frequency \emph{bins}, so a whole-prefix
transform would let the cutoff period drift with context length ($\ge$256
tokens at $T{=}4096$ but $\ge$32k tokens at 512k), changing what ``anomalous''
means as $S$ grows. Fixed windows keep $\sigma$'s meaning scale-invariant,
make it naturally incremental (each block is transformed exactly once, at
close; its $\sigma$ is immutable), and let prefill and decode-time closes
share one bit-identical implementation.
We keep the $\sigma$-top-$m$ rows (plus 4 sinks and the recent window) in the attended
tier; move every other row, bit-exactly, to the archive of
Section~\ref{sec:tier}. \emph{No row is ever deleted in deployment}: the
invariant is a partition (pure deletion appears only as an ablation), and
it is a cache-manager policy, not a model change --- the selector reads
the trailing 64 dims of each row, a contiguous slice. The partition stays
live during decode: every 4096 generated tokens the newest block closes,
its $\sigma$ joins the stored record, and one global top-$m$ over all
closed rows rebalances the tier; evicted rows enter the archive through
cached projections, so the invariant holds during generation exactly as at
prefill (live-archive gate: parity PASS, $\Delta$CE median \dceVkMed{}
nats, needle 8/8).

\dividend{the policy consults no observed attention and never revisits a
decision: the signal exists query-free (Section~\ref{sec:math}a) and frozen
rankings never stale (Corollary~\ref{cor:stationary}).}{the
branch signal is phase-buried (measured 0.08 branch-alone) and any frozen
ranking goes stale at the frequency-band rate.}

\subsection{The recall tier}\label{sec:tier}
Eviction discards rows whose importance arrives later; production systems
keep them recallable \citep{chen2024arkvale}. VestigeKV's trigger is
\emph{query-adaptive, not quota}: each step's fired set is the archived
rows that, after certified inflation, could still beat the kept tier's
best for \emph{this} query --- empty when the query is well served, the
minimal sufficient set for the recall target $\tau$ when its target was
evicted. Fixed budgets under- and over-fetch by construction; a decision
does neither. Archived row $u$ keeps an index entry $(r_u,\;V_r^\top\hat c_u,\;\eta_u)$:
the \emph{exact} branch summand of every future logit, a rank-$r$ sketch of the
content summand ($V_r$ from the context's own prefix queries), and the sketch
residual norm $\eta_u=\lVert(I-V_rV_r^\top)\hat c_u\rVert$. Per decode step, over archived rows:
\begin{equation}
\mathrm{score}(t,u)=\underbrace{\lambda\,q^{r\top}_t r_u}_{\text{exact summand}}
+\underbrace{\lambda\,(V_r^\top q^{c\prime}_t)^\top(V_r^\top\hat c_u)}_{\text{sketch}}
+\;z\,\underbrace{\lambda\,\lVert(I{-}V_rV_r^\top)q^{c\prime}_t\rVert\,
\eta_u/\sqrt{d_c{-}r}}_{\text{certificate scale}},
\label{eq:trigger}
\end{equation}
with $q^{c\prime}_t=W_{UK}^\top q^c_t$ the absorbed content query,
$d_c{=}512$ the content-latent width, and the
$1/\sqrt{d_c{-}r}$ factor putting the residual on a per-dimension scale before
calibration; fetch the top-$j$ rows whose score exceeds the tier-1 maximum. $z$
has a closed form: over hard calibration queries (those whose causal-argmax row
is archived) the required inflation is $z_q = (\max_1 - \mathrm{idx}_t)/\mathrm{cert}_t$,
and $z$ is the conformal quantile $z_{(\lceil (n{+}1)\tau_s\rceil)}$ of these,
which by exchangeability alone guarantees $P(\text{target recovered}) \ge \tau_s$
on a fresh hard query --- no distributional assumption, no external calibration
corpus (labels are computable at compression time, before any row is deleted).
The scan level $\tau_s = \tau/(1-\alpha)$ and the gate's false-close budget
$\alpha = (1-\tau)/2$ both derive from the single recall target $\tau$ by an
even split of the miss budget, and the minimum calibration sample size
$\lceil \tau_s/(1-\tau_s)\rceil$ is the smallest $n$ for which that order
statistic exists --- the tier has exactly one quality knob.

\begin{figure}[H]
\centering
\begin{tikzpicture}[font=\small,x=1cm,y=1cm,
  box/.style={draw,rounded corners=2pt,align=center,inner sep=5pt}]
  \node[box,fill=gray!10,minimum width=3.1cm] (kept) at (0,1.5) {tier 1: kept rows\\$m\times576$, exact};
  \node[box,fill=orange!15,minimum width=3.1cm] (idx) at (0,0) {index\\$(T{-}m)\times(64{+}r)$};
  \node[box,fill=blue!8,minimum width=2.6cm] (attn) at (4.1,1.5) {attention\\$\max$ logit $s^{*}$};
  \node[box,fill=orange!25,minimum width=2.6cm] (scan) at (4.1,0) {scan, Eq.~\ref{eq:trigger}\\top-$j>s^{*}$?};
  \node[box,fill=blue!8,minimum width=2.2cm] (out) at (8.0,1.5) {merged\\softmax $o_t$};
  \node[box,fill=gray!25,minimum width=2.6cm] (arch) at (8.0,0) {tier 2: archive\\$(T{-}m)\times576$, exact};
  \node at (-2.15,1.5) {$q_t$};
  \draw[->] (-1.9,1.5) -- (kept);
  \draw[->] (-1.9,1.5) |- (idx);
  \draw[->] (kept) -- (attn);
  \draw[->] (idx) -- (scan);
  \draw[->] (attn) -- (scan) node[midway,right] {\scriptsize $s^{*}$};
  \draw[->] (scan) -- (arch) node[midway,above] {\scriptsize fetch ids};
  \draw[->] (arch) -- (out) node[midway,right] {\scriptsize rows};
  \draw[->] (attn) -- (out);
  \draw[dashed] (6.55,-0.8) -- (6.55,2.35);
  \node[anchor=west] at (6.6,2.5) {\scriptsize archive tier (GPU)};
  \node[anchor=east] at (6.5,2.5) {\scriptsize attended tier (GPU)};
\end{tikzpicture}
\caption{Per decode step: tier-1 attention yields $s^{*}$; the index scan admits the
archived rows that could compete; fetched rows join that query's softmax as exact
copies (Lemma~\ref{lem:exch}). Both tiers are GPU-resident: the archive is
not a demotion to slower storage but a cheaper \emph{read discipline} over the
same memory --- the per-step saving is attention bytes, and the deployment form
is exactly this two-tier loop.}
\label{fig:arch}
\end{figure}
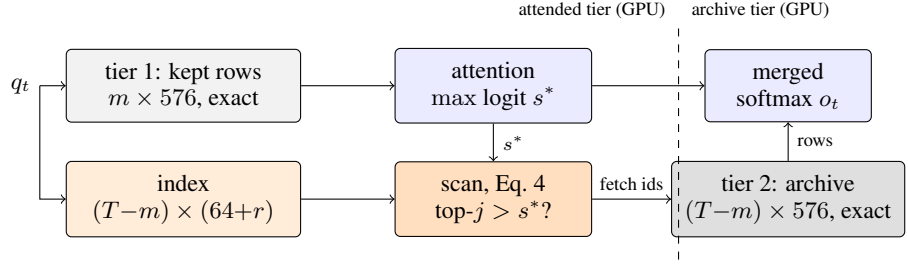

Two lemmas make the construction precise.

\begin{lemma}[Index certificate]\label{lem:cert}
For every future query and archived row,
$s(t,u)=\lambda q_t^{r\top}r_u+\lambda(V_r^\top q_t^{c\prime})^\top(V_r^\top\hat c_u)
+\delta_{t,u}$ with
$|\delta_{t,u}|\le\lambda\lVert(I-V_rV_r^\top)q_t^{c\prime}\rVert\,\eta_u$.
The first two terms are computable from the index alone; the bound is
Cauchy--Schwarz on the sketch residual and holds for \emph{all} queries because the
form is fixed --- under RoPE it would have to hold over the rotation orbit. \qed
\end{lemma}

\begin{lemma}[Bounded leakage]\label{lem:leak}
Let $F$ be the rows admitted to a step's softmax (kept $\cup$ fetched) and
$s^{*}=\max_{u\in F}s(t,u)$. If every excluded row satisfies
$s(t,u)\le s^{*}-\gamma$, the total excluded softmax weight is at most
$(T-|F|)\,e^{-\gamma}$ relative to the maximal included weight, and the output error
is bounded by that mass times $\max_u\lVert v_u\rVert$.
\end{lemma}
\begin{proof}
Each excluded weight is at most $e^{-\gamma}$ times the maximal included weight; sum
over excluded rows and bound the output difference by excluded mass times the
largest value norm.
\end{proof}

Lemma~\ref{lem:cert} is what the scan computes; Lemma~\ref{lem:leak} is why fetching
only score-competitive rows suffices: the trigger exists to make the premise of
Lemma~\ref{lem:leak} hold with high measured probability (\trigConfRec{} in
distribution). Guaranteed variants were measured and closed: the sound
Cauchy--Schwarz bound fires on \trigSoundRate{} of steps, and Gaussian-derived
thresholds are invalidated by a selection effect (the argmax rows are precisely the
residual outliers, at \trigGaussSigma). The adopted $z$ is therefore the
distribution-free conformal quantile of the empirically required inflations ---
assumption-light where the Gaussian route was assumption-broken.

\dividend{the index owes its 64 exact dims and its per-row certificate to a
sidecar term that nothing rotates; both quantities are single numbers per
row.}{each would have to hold over the whole rotation orbit ---
the obstruction quantified next.}

% AUTO-GENERATED by mexp/update_serving_macros.py -- do not hand-edit
\newcommand{\srvBatchOne}{1.12}  % bs=1: 138.1/123.7 tok/s
\newcommand{\srvBatchTwo}{1.09}  % bs=2: 140.1/128.4 tok/s
\newcommand{\srvBatchFour}{1.15}  % bs=4: 233.1/203.5 tok/s
\newcommand{\srvBatchEight}{1.20}  % bs=8: 347.0/288.1 tok/s
\newcommand{\srvBatchSixteen}{1.25}  % bs=16: 490.7/391.6 tok/s
\newcommand{\srvBatchList}{$1.12\times$, $1.09\times$, $1.15\times$, $1.20\times$, $1.25\times$}

% Additive fragment for sec:eng — the production serving-path implementation
% (standard sglang port). Reviewed-by-author before \input into main.tex.
% Numbers: out/final_vk_512k.jsonl / final_base_512k.jsonl (continuous
% decode) + out/tput_*_bs*.jsonl (sglang.benchmark.serving; sec:audit).

\subsection{The attention implementation on a production stack}\label{sec:serving}

Throughout, \emph{k} denotes $1024$ tokens (so 4k prefill is $4096$ tokens, 64k is $65536$, and the 512k latency sweep reaches $524288$).
The mathematics of Section~\ref{sec:math} buys a specific engineering
property, and this subsection makes it concrete: we port VestigeKV into
standard \texttt{sglang} as an attention backend
(\texttt{--attention-backend vestigekv\_mla}) and serve Kimi Linear 48B
across two GPUs (pipeline parallel), with the eviction running under CUDA
graphs at zero per-step host cost. No kernel is written: the backend wraps
the stock fused MLA-decode kernel and changes only \emph{which rows it
reads}. The backend enforces its own precondition rather than trusting the
operator: selection is refused unless the model is a NoPE-MLA cache
(\texttt{mla\_use\_nope}), because a positional encoding on the decoupled
branch destroys the eviction signal (Table~\ref{tab:rope}); it also refuses
\texttt{page\_size}${\neq}1$ and speculative decoding, which the kept-index
tables do not yet address. A config that fails any check raises rather than
degrades silently.

\paragraph{Decode step: seven kernels inside the model graph.} Per layer
the attended set is $\mathrm{kept} \cup \text{tail} \cup
\text{fired}(q_t)$, both tiers always on. The whole recall step ---
online-softmax prologue over kept rows, batched archive scan (native
bf16/fp16 tensor-core dots), deterministic compaction, and the all-layer
CSR pack --- runs as \emph{seven fused Triton kernels baked into the
decode CUDA graph}: one \texttt{cudaGraphLaunch} per step, no second
graph, no VestigeKV-triggered recapture (worst-case-capacity buffers
refresh in place). Calibration fits on the request's first decode queries
with a provisional index serving from step one; the $O(S)$ calibrated
build runs on a side stream. At $S{=}512$k the trigger rate is $99.8\%$
per step: recall is a necessary component, and no production off-switch
exists. Pre-fusion implementations remain in-tree as bit-parity oracles;
the fused pipeline's greedy output is token-identical to the unfused one.

\paragraph{Speedup accounting: what this buys, and where.} Decode step time decomposes as $t(S) = C_0 + T_{\mathrm{net}} + A(S)$,
and attention $A(S)$ is the only term any selection scheme can shrink. On
this deployment (two GPUs over ethernet pipeline), the measured
decomposition is $C_0 \approx 5.5$\,ms (MoE with 3B active parameters, 20
linear-attention layers, sampling), $T_{\mathrm{net}} \approx 2.0$\,ms
(four NCCL hops per step, measured), and
$A_{\mathrm{dense}} \approx 0.57$\,ms at $S{=}64$k by measured-slope
calibration --- attention is $7\%$ of the step there. VestigeKV replaces
$A_{\mathrm{dense}} = kS$ with $k\rho S$ (kept attention) $+\,k_s(1-\rho)S$
(the archive scan at $260$\,B/row with the 16-bit index) $+\,R$, where the
fused pipeline's fixed adder measures $R$ = \vkTaxUs{}, batch-independent
(the 4k-context arms measure even). The $S$-slope ratio is
$\rho + k_s/k \approx 0.26$: the scan must touch every archived row every
step --- what the recall guarantee costs --- so the attention slope drops
to a quarter, not to zero.

\paragraph{What it measures.} Measured attention slopes: ${\sim}0.7$ vs.\ $7.6$\,ns
per step per cached token
--- below the byte account, for a measured reason (the dense arm's long-$S$ slope
exceeds its own bandwidth floor via split-count growth, which vestigekv
sidesteps by holding its attended set near $\rho S$). Crossover at
\serveCrossover{}; \serveSpeedTwoFiveSix{} at 256k, \serveSpeedMax{} at
496k: \textbf{this is a long-context method}, deployment-dependent in the
direction that understates the method --- this testbed's large $C_0$ and $T_{\mathrm{net}}$
\emph{suppress} the ratio; single-device or NVLink lifts the table.

\paragraph{Where it does not buy anything.} The \emph{start} of a request is its slowest
stretch --- the adaptive calibration window (8--64 decode steps) fits the
conformal certificate on live queries, a per-request cost bounded by the
window and independent of $S$ (the first bucket runs ${\sim}0.2$\,ms/step
above steady state; gone well before 112k). It is \emph{not} what keeps
bs=1 throughput near dense at 64k$+$4k (the transient is under $0.3\%$ of
a 4k-token decode): that near-parity is the steady-state attention share
--- $7\%$ of the step at 64k --- where compression's saving is roughly
offset by the recall scan; the advantage appears where attention's share
grows, with batch and with context length. Two scope notes: the advantage is \emph{speed, not memory} (both tiers
GPU-resident; nothing is deleted; host-offload is future work), and
launch-to-launch network variance is $\pm 15\%$ on the fixed terms, so
each arm's curve is one launch and cross-arm claims quote $\Delta$-vs-4k.
A stage-separated profile confirms the account at kernel level (attention
$8.8\%\to20.8\%$ of GPU-busy from batch 1 to 16; the compressed kernel
shrinks $4.3\times$--$15\times$).
One property classic sparse attention does not have: the attention volume
is not a fixed budget. The tier-1 set is \emph{proportional} ($\rho T$
plus the tail; the only schedule measured quality-flat in length), and the
fired set is query-dependent, zero to a measured worst case of 3{,}983
rows (the theoretical worst case attends the whole cache exactly ---
quality unharmed, one step's latency degenerates). An \emph{optional} cap
--- ArkVale's budgeted top-$j$ on our certified scores --- hard-bounds
per-step work where a latency SLO needs it; the mainline is uncapped.
This MoE-dominated deployment is the \emph{conservative} end of the
envelope: on an attention-dominated stack (dense decoder, single device)
the same slope ratio acts on a far larger share of the step, and the
end-to-end ratio approaches the attention-only account of
Section~\ref{par:reads}. Both arms of every comparison run the same
graphs, and every serving number comes from \texttt{sglang}'s own
benchmark module (Section~\ref{sec:audit}).

\dividend{the kept set is frozen before any query exists, so the sparse
pattern is a compile-time constant of the decode graph: selection costs
zero on the token path.}{a RoPE selector must consult the live query,
forcing an out-of-graph rebuild per step --- the pattern cannot be
captured.}

The digest obstruction under RoPE --- why ArkVale/Quest-style page summaries are structurally loose under rotation and tighten under NoPE --- is detailed in Appendix~\ref{sec:arkvale}.

\section{Experiments and recommended configuration}\label{sec:exp}
\textbf{Protocol.} Every serving number in this paper follows the
\emph{continuous-generation} protocol, because that is the deployment form:
bs=1 latency is one request prefilled at 4k and decoded continuously to
512k (decode-time block closes and the live archive run throughout, so each
step's cost is the system's true per-step cost at that length), and
throughput is 64k-prefill/4k-decode requests swept over batch size. No
fixed-decode-at-a-prefill-length point is quoted as a performance claim.
Both arms run from one source tree --- the baseline is the wrapped base
itself (\texttt{--attention-backend triton}) --- each arm's curve is
collected within a single launch, and every launch log prints the same
content-level weight fingerprint. Quality tables report the deployed
two-tier form; the tier-1 ablation is demoted to
Appendix~\ref{app:tier1} (recovery \ttxOTE{} at $128\times$).

\begin{figure}[H]
\centering
\begin{minipage}[t]{0.47\linewidth}
\includegraphics[width=\linewidth]{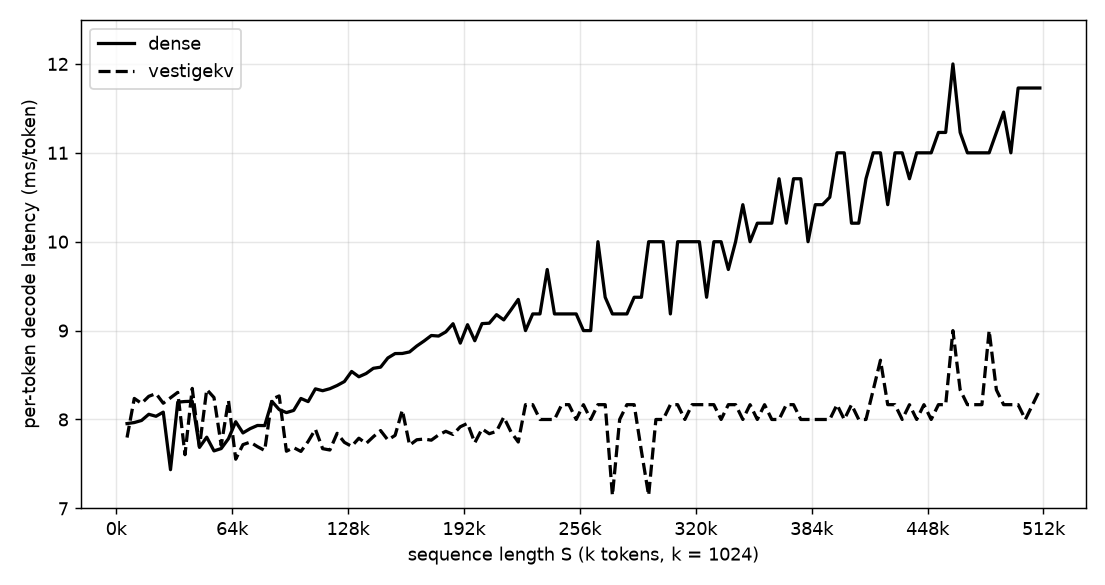}
\end{minipage}\hfill
\begin{minipage}[t]{0.445\linewidth}
\includegraphics[width=\linewidth]{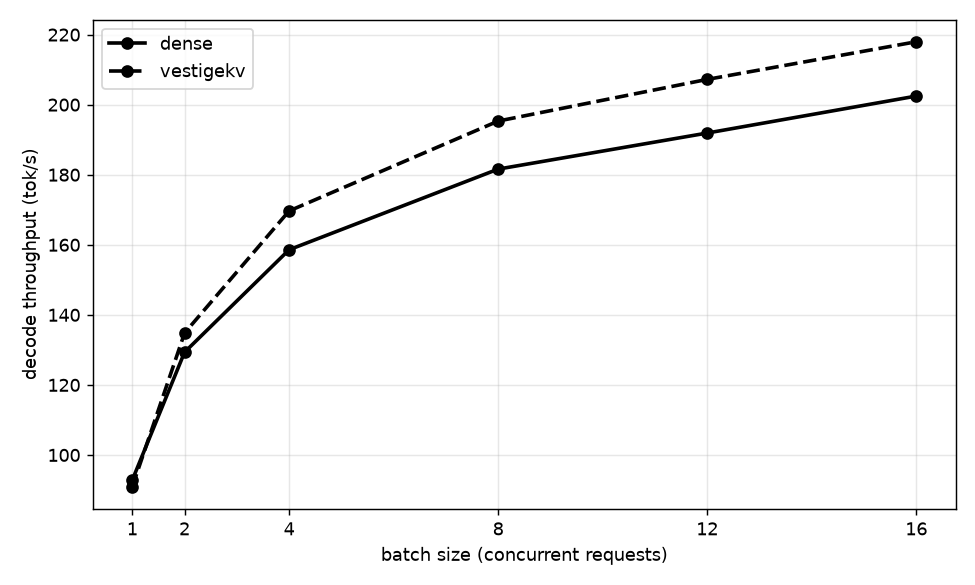}
\end{minipage}
\caption{\small\textbf{Left:} bs=1 continuous decode, one launch per arm,
one statistic (median per 4096-token bucket; client records to 224k,
server decode log beyond; splice invisible by construction). Crossover at
\serveCrossover{}; the vestigekv attention slope is flat at ${\sim}0.7$\,ns
per step per cached token vs.\ dense $7.6$; \serveSpeedTwoFiveSix{} at 256k, \serveSpeedMax{} at 496k. The
left-edge elevation is the calibration transient; spikes are the
block-close sawtooth. \textbf{Right:} throughput, 64k prefill $+$ 4k
decode, bs $\le$ 16 (memory cap): the advantage grows with batch
(attention does not amortize across the batch; shared weights do) --- launch-noise even at
bs=1, \tputGainTwelve{} at bs=12.}
\label{fig:curve}
\end{figure}

Primary metric: needle-retrieval intact rate ($\Delta$NLL $<1$ nat), uncompressed
baseline gated (all trials passed); gates and process in
Section~\ref{sec:audit}. Per-layer proxies are demoted: in \nProxyInversions{} recorded
cases a per-layer proxy mispredicted end-to-end retrieval, the worst a learned selector at $+0.41$
offline / $-0.71$ end-to-end.

% (tier-1 ablation tables moved to App.~\ref{app:tier1} per the deployment-first revision)
Three further results, tables in Appendix~\ref{app:exp_tables}: (i) the
frozen dual-arm control --- the identical operator collapses on RoPE
(\vkEightKxThirtyTwo{} NoPE vs.\ \ropeVKxThirtyTwo{} RoPE at $32\times$;
Table~\ref{tab:rope}); (ii) the recall tier restores what eviction loses
(\vkEightKxThirtyTwo\,$\to$\,\ttxThirtyTwo{} at $32\times$,
\vkEightKxOTE\,$\to$\,\ttxOTE{} at $128\times$) while admitting only
\ttMaskExtra{} extra rows per query, and a larger sketch rank cuts fire
rate while recall \emph{rises} (Table~\ref{tab:tier}); (iii) teacher-forced
$\Delta$CE is effectively lossless at the standard configuration
(${\sim}0.0006$ bits/byte at $32\times$), while the recent-only floor
matches on CE yet degrades the needle answer by $+15$ nats of NLL --- continuation loss cannot
arbitrate retrieval, which is why needle is primary
(Table~\ref{tab:bpb}).

At $8\times$ retrieval is lossless (\vkEightKxEight); quality at fixed ratio
\emph{improves} with length (\vkEightKxOTE{} at 8k $\to$ \vkThirtyTwoKxOTE{} at 32k,
$128\times$).

\subsection{The deployment path measures the same}\label{sec:deploy}
The tables above come from an HF-forward harness (\texttt{harness/} in the
code release) that records exactly which rows selection and recall keep and
fetch. Two links tie it to the deployed system, both reproducible from the
release. \emph{Selection parity}: the fused kernels' fired-and-fetched row
sets are bit-identical to the harness's on shared inputs (registered
parity tests), and a per-step row-set assertion holds over a full
benchmark (Section~\ref{sec:audit}). \emph{Serving quality}: on the
serving path with recall live every step, VestigeKV is statistically
indistinguishable from dense --- Table~\ref{tab:quality}: 64-shot gsm8k
\gsmVk{} vs.\ \gsmBase{}, MAUVE \mauveVkFused{} vs.\ \mauveEngFused{}
(shared per-context seeds), teacher-forced $\Delta$CE median \dceVkMed{}
nats (Table~\ref{tab:quality}). Because the fired set is bit-identical between
harness and port, the retrieval and language-modeling numbers transfer to
the deployment unchanged.

\subsection{Recommended configuration}\label{sec:config}
\textbf{One quality knob, everything else derived or structural.} The
deployment exposes exactly one quality parameter, the recall target
$\tau=0.90$; the rest of the trigger calibration is a closed-form chain:
$\alpha=(1-\tau)/2$ (the gate's false-close budget),
$\tau_s=\tau/(1-\alpha)$ (the scan-level target that leaves room for the
gate), $z = z_{(k)}$ with $k=\lceil (n{+}1)\tau_s\rceil$ (the conformal
quantile, closed form over the calibration scores),
$\mathrm{min\_hard}=\lceil \tau_s/(1-\tau_s)\rceil = 18$ (the smallest
calibration set that can certify $\tau_s$), an $8\!\to\!64$ adaptive
calibration window, and a $Z_{\max}=8$ clamp. Distinct from the knob are
the \emph{structural constants} --- not tuned per deployment, each with a
stated reason: $\rho=1/32$ (the compression ratio, an experimental axis),
$\kappa=16$ (low-pass bandwidth), $r=64$ (index rank), 4 sinks, a
one-block recent window, and the 4096-token close granularity that anchors
$\sigma$'s scale invariance.

An optional per-layer index cascade preserves recovery exactly and cuts
the scan term to $0.65$--$0.79\times$ (its uniform variant is refuted;
Appendix~\ref{app:tier1}).

The ratio floor for a constant-$m$ schedule is $1/32$ --- the only
measured ratio flat in length (Table~\ref{tab:main}; deeper ratios degrade
precisely at long context, where a floor binds); tier-2 recovery extends
depth at length (\mFiveTKTier{} at an instrumental $256\times$ in the
deployment loop).
Every default traces to a pre-registered measurement; the full table
(each value with its basis) is Appendix~\ref{app:exp_tables},
Table~\ref{tab:config}. The load-bearing facts: $\rho{=}1/8$ is lossless
and $1/32$ the default; the sketch rank stays at $r{=}64$ because the
GPU-resident archive makes fetches free (larger $r$ buys only a lower fire
rate); the fetch is uncapped by default (the fired set is the
certificate's decision, not a quota; $\mathrm{topj}{=}16$ is an optional
worst-case-latency bound); the trigger $z$ and the entropy gate derive in
closed form from the single target $\tau{=}0.90$.

\dividend{the entire configuration self-calibrates from the context being
compressed --- possible because labels, thresholds and the sketch basis are all
functions of one fixed bilinear form.}{calibration targets move with query
position; every threshold must chase a distribution that the compression step
cannot observe.}

\section{Measurement integrity}\label{sec:audit}
Every gate was driven with a known-bad input and shown to refuse it before
being trusted; every threshold was frozen, dated, before its data; all
wall-clock numbers come from \texttt{sglang}'s own benchmark module on one
source tree, with a per-decode-step attended-row-set assertion proven on
both arms (zero violations) and a content-level weight fingerprint on
every launch. The full protocol --- gates, Wilson intervals, the row-set
invariant, and the serving methodology --- is Appendix~\ref{app:integrity}.

\section{Related work}
Trained sparse attention \citep{yuan2025nsa,lu2025moba} buys
content-aware patterns with pre-training; VestigeKV inherits its pattern
from a channel that NoPE training has already repurposed, at zero training
cost --- the kinship is structural, the provenance is not.
Observed-attention eviction
\citep{zhang2023h2o,li2024snapkv,oren2024tova,xiao2024streamingllm} selects on
attention already seen; Section~\ref{sec:problem} measures the setting where that
information does not yet exist. Branch-differential studies
\citep{ma2026irminsul,ma2026kamera} measure the \emph{rotated} branch and report
that rotation destroys its position-free structure (branch-only AUC $0.43$),
leaving the unrotated case open --- the case measured here. ArkVale \citep{chen2024arkvale} deserves the closest comparison: the
evict-summarize-recall architecture is theirs, and VestigeKV inherits it.
Four differences, one theoretical. (i)~\emph{Signal}: query-history page
importance vs.\ a query-independent signal that exists before any query
does (Section~\ref{sec:problem}, where history has nothing to read).
(ii)~\emph{Granularity}: page-level lossy digests vs.\ row-level
bit-exact latents plus a 16-bit index. (iii)~\emph{Admission}: a top-$k$
page \emph{quota} vs.\ an \emph{uncapped} calibrated competition against
the query's kept-tier maximum --- the fired set is a per-query decision,
zero to everything (their budgeted rule survives in ours as an optional
worst-case-latency bound, Section~\ref{sec:serving}).
(iv)~\emph{Certificate}: a digest bounds page relevance only up to a
bounding volume valid over the whole rotation orbit --- under RoPE no
tight query-uniform bound exists (Appendix~\ref{sec:arkvale},
Table~\ref{tab:arkvale}) --- whereas the fixed NoPE bilinear form makes
the per-row Cauchy--Schwarz bound query-uniform, which a conformal
threshold can then calibrate; stationarity
(Corollary~\ref{cor:stationary}) additionally freezes the ranking into a
capture-time constant. ArkVale is the right architecture attached to a
score family where no tight certificate can exist; NoPE is where it
acquires a proof.
Page-bound selection \citep{tang2024quest} inherits the same digest
obstruction; trained compression
\citep{lin2025matryoshkakv,gelberg2026kvcat} changes the model, which VestigeKV
does not. Per-step exact merging under RoPE \citep{tian2025keepkv} is the strongest
possible there; Proposition~\ref{prop:impossible} shows the query-universal version
exists only under NoPE$+$MLA. Architecture from \citep{deepseek2024v2}; the NoPE
deployment from \citep{kimi2025linear}; NoPE's length-generalization case from
\citep{kazemnejad2023nope,wang2024lengthgen}. Appendix~\ref{app:dividends} derives the dividend each cited
method inherits under NoPE.

\section{Limitations}
Scope boundaries (full table: Appendix~\ref{app:limits}): one measured
model, Kimi Linear 48B Base --- the only released NoPE-MLA checkpoint
(K3's Gated-MLA is covered by gate-commutation; its salience content is
untested); selector ordering replicates on a NoPE GQA hybrid, eviction
depth does not (depth is MLA-conditional); the depth's mechanism is open;
per-step consultation is probabilistic (0.9--1.0, 1.00 end to end); long-
context suites are future work.

\section{Outlook: training against this mathematics}\label{sec:outlook}
Everything above is post-hoc. If training is allowed, the theory marks
prospects (not claims): a trained index channel could make sketch residuals
small by construction, reviving \emph{guaranteed}-recall sparse attention;
training toward position-free latents on repeated spans would make
Proposition~\ref{prop:impossible}'s collapse live, taking cache size to
$O(\text{distinct content})$; and eviction-aware fine-tuning
\citep{gelberg2026kvcat} has a uniquely stable target under NoPE. All
conditional on scale.

\section{Conclusion}
A production NoPE-MLA model keeps, in a channel its architecture no
longer needs, a trained record of which tokens matter: reading it costs
\readCost{} of the cache, acting on it compresses $8$--$32\times$ with
retrieval intact, a recall tier restores $128\times$, and none of it
survives rotation. The vestige is the signal.

\ifdefined\NAMEDBUILD\else
\section*{The Use of Large Language Models}
In accordance with the ICLR policies on LLM usage, we disclose that LLMs
were used substantially: (i)~coding and experiment orchestration for the
serving-stack implementation and benchmark campaigns; (ii)~proof drafting
and checking; (iii)~literature search; (iv)~drafting and polishing the
writing. LLMs were not used to fabricate or select results, and no claim
entered the paper on LLM output alone: the research direction, the
pre-registered decision rules (frozen before data), and final verification
are the authors' --- every proof was checked manually, every reported
number regenerates from archived run records, and the gates of
Section~\ref{sec:audit} were driven with known-bad inputs before being
trusted. The authors take full responsibility for all content.

\section*{Reproducibility statement}
The supplementary archive contains: \nVerdicts{} pre-registration files with
decision rules dated before their data and verdicts appended after, including
retracted predictions; every run record (JSON) behind every number, with the
script that regenerates all numeric macros in this
paper from those records; the full harness with its gates; and the exact
commands, seeds, and ordering of every GPU run. Raw activation dumps are
excluded for size and are regenerated by the included extraction script from
the public checkpoints.

\fi

\bibliography{refs}
\bibliographystyle{iclr2027_conference}

\appendix
\section{Proofs}\label{app:proofs}
Lemma~\ref{lem:exch} is proved where stated. Full proof of
Proposition~\ref{prop:impossible}:

\begin{proof}[Proof of Proposition~\ref{prop:impossible}]
Suppose $\sum_g e^{b_g}e^{\langle Q,C_g\rangle}v(C_g)
=\sum_u e^{\langle Q,\tilde C_u\rangle}v(\tilde C_u)$ and the same identity without
the $v$ factors, for all $Q$ in an open set. For pairwise-distinct $C_i$, the
functions $Q\mapsto e^{\langle Q,C_i\rangle}$ are characters of $(\mathbb{R}^d,+)$
and hence linearly independent on any open set; matching coefficients forces the
representative multiset to contain every distinct row with $e^{b_g}$ equal to its
multiplicity, so $B\ge\#\{\text{distinct rows}\}$. Under RoPE the displayed distance
bound makes all $T$ rows pairwise distinct. Under NoPE$+$MLA, if $c_v=\alpha c_u$
($\alpha>0$) and $r_v=r_u$, positive homogeneity of RMSNorm gives
$\hat c_v=\hat c_u$, hence identical rows: the merge class strictly contains vector
equality. The value side is consistent because $v$ is a function of the row --- in
standard attention, with $V$ cached independently, no such consistency holds and
the numerator identity is not even well-posed.
\end{proof}

\ifdefined\NAMEDBUILD\else

\fi

\section{The digest obstruction under RoPE}\label{sec:arkvale}
The prior recall design \citep{chen2024arkvale} scores an evicted page $K$ by a
bounding-sphere digest, $I(q,K)=\max_{k\in K}q\cdot k\le q\cdot c+r\lVert q\rVert$,
sound only when $r$ encloses every key --- and adopts a non-enclosing radius because
the sound one overestimates, giving up the guarantee. Under RoPE that dilemma is
structural. A page is $P$ consecutive positions of cached keys $R_uk_u$; even for
\emph{constant} content $k$,
\begin{equation}
\lVert R_uk-R_vk\rVert^2=\sum\nolimits_j 4\sin^2\tfrac{\theta_j(u-v)}{2}\,
\lVert k^{(j)}\rVert^2
\;\Longrightarrow\;
r^2\;\gtrsim\;\sum\nolimits_{\theta_j\ge 2/P}\lVert k^{(j)}\rVert^2,
\label{eq:radius}
\end{equation}
so the radius is of order the fast-frequency key norm regardless of content
coherence: the digest cannot be tight. The center fares no better --- per frequency
pair the page mean carries the Dirichlet factor
\begin{equation}
\Big|\tfrac1P\sum\nolimits_{u=0}^{P-1}e^{i\theta_j u}\Big|
=\frac{|\sin(P\theta_j/2)|}{P\,|\sin(\theta_j/2)|}\;\ll\;1
\quad\text{for fast }\theta_j,
\label{eq:center}
\end{equation}
so $q\cdot c$ is blind to exactly the dimensions inflating $r$. NoPE removes both at
the root --- rows are content-only --- and replaces the estimate with structure the
digest never had: an exact summand (Lemma~\ref{lem:cert}), a per-token certificate
with a measured accuracy knob, and by Section~\ref{sec:math}(a) a fixed winner set,
so rankings may be prefetched where the RoPE family $A_h(\Delta)$ forces re-ranking
every step.

\begin{table}[H]
\centering\small
\caption{Each ArkVale-style limitation, its mechanism, and the measured effect of
removing rotation. Same digest protocol on both models (32-token pages, sphere
digests per their Eq.~1--2); recall@8 pages of the true-argmax page, worst layer ---
the tail-risk statistic that decides adoptability.}
\label{tab:arkvale}
\begin{tabular}{l>{\raggedright\arraybackslash}p{4.0cm}cc}
\toprule
limitation & mechanism & RoPE & NoPE\\
\midrule
digest looseness & Eq.~\ref{eq:radius}: orbit-inflated radius
 & \spreadInfl{} spread & (removed)\\
worst-layer recall (heuristic $r$) & deep layers position-dominated
 & \pdDsvWorst & \pdKimiWorst\\
worst-layer recall (sound $r$) & sound radius unusable
 & \pdDsvSoundWorst & \pdKimiSoundWorst\\
per-step re-ranking & winner set sweeps with position
 & \unionDsv{} union & \unionKimi{} union\\
\bottomrule
\end{tabular}
\end{table}

\dividend{exact summands, per-token certificates, layer-uniform recall, and
prefetchable rankings --- the parts of a recall tier that mathematics can supply.}{
every one of these reduces to a heuristic estimate; the residual risk hides in
specific layers (\pdDsvWorst{} at the worst), which is the adoption-blocking tail
risk.}

\section{Experiment tables (harness)}\label{app:exp_tables}

\begin{figure}[H]
\centering
\begin{tikzpicture}[x=0.017cm,y=1cm,font=\small]
  \node[anchor=west] at (0,0.95) {cache row $\tilde C_u$ (576 dims), one per token per layer};
  \draw[fill=gray!12] (0,0) rectangle (512,0.55);
  \draw[fill=orange!25] (512,0) rectangle (576,0.55);
  \node at (256,0.28) {content latent $\hat c_u$: 512 dims, RMSNormed (direction only)};
  \node at (544,0.30) {$r_u$};
  \draw[->,thick,orange!70!black] (544,-0.06) -- (544,-0.50);
  \node[anchor=north east] at (576,-0.55) {$\sigma_u=\lVert r_u-\text{low-pass}(r)_u\rVert
    \;\Rightarrow\;$ keep top-$m$ rows exactly, evict the rest};
\end{tikzpicture}
\caption{VestigeKV reads only the branch (orange, \readCost{} of the row). Kept rows
are exact; nothing else is stored.}
\label{fig:method}
\end{figure}
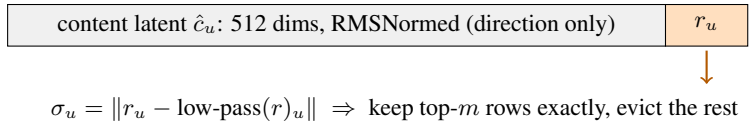

\begin{figure}[H]
\centering
\begin{tikzpicture}[x=0.05cm,y=0.8cm,font=\footnotesize]
  \draw[->] (0,0) -- (205,0) node[right] {$t$};
  \draw[fill=gray!15] (0,0.10) rectangle (100,0.50);
  \draw[fill=orange!40] (36,0.10) rectangle (40,0.50);
  \draw[fill=blue!12] (100,0.10) rectangle (200,0.50);
  \node at (150,0.30) {future queries};
  \node at (17,0.30) {prefix};
  \node[orange!60!black,anchor=north] at (30,-0.10) {needle};
  \draw[orange!60!black] (34,-0.10) -- (38,0.08);
  \draw[very thick,red!70!black] (100,-0.55) -- (100,0.60);
  \node[red!70!black,anchor=north] at (108,-0.60) {compression};
  \draw[->,thick,blue!60!black] (168,0.58) .. controls (120,1.45) .. (40,0.62);
  \node[blue!60!black,anchor=south] at (112,1.22) {the attention that makes the needle matter};
\end{tikzpicture}
\caption{The needle's importance is conferred entirely by queries that do
not exist at compression time; H2O accumulates prefix attention, SnapKV
reads a window before the cut --- both see nothing
(Table~\ref{tab:baselines}).}
\label{fig:timeline}
\end{figure}
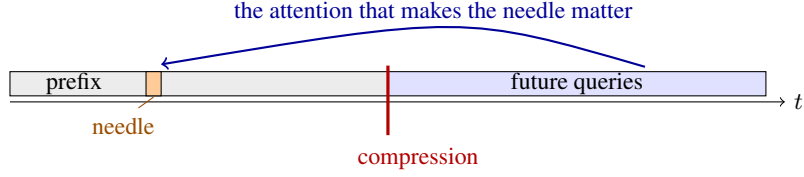

\begin{table}[H]
\centering\small
\caption{Deployed-path quality on the final fused code (bars frozen before
data; both numbers verbatim). gsm8k runs 64-shot so the prompt
(${\sim}9$k tokens) crosses the close granularity and compression is
genuinely engaged; MAUVE generations share per-context sampling seeds
across arms. Units: accuracy gaps read in percentage points, MAUVE in raw
score (a divergence-derived quantity; relative percentages are not
meaningful for either), and only the $\Delta$CE row carries a natural
relative reading ($+0.17\%$ of the model's bits/byte in aggregate:
\bpbEng{}$\to$\bpbVk{}).}
\label{tab:quality}
\begin{tabular}{lccc}
\toprule
 & dense (triton) & vestigekv & bar\\
\midrule
gsm8k 64-shot ($n{=}800$) & \gsmBase & \gsmVk & $\ge$ dense $-0.07$ (2$\sigma$)\\
MAUVE (16 ctx, gpt2-large) & \mauveEngFused & \mauveVkFused & $\ge$ dense $-0.10$\\
teacher-forced $\Delta$CE (nats) & \multicolumn{2}{c}{\dceVkMed{} median (bit-equivalence transfer)} & $\le 0.05$\\
\bottomrule
\end{tabular}
\end{table}

\begin{table}[H]
\centering
\footnotesize\setlength{\tabcolsep}{4pt}
\begin{tabular}{lll}
\toprule
measurement & NoPE (Kimi) & RoPE (DSV2)\\
\midrule
branch/content row norm (max) & \rowNormMaxKimi$\times$ & \rowNormMaxDsv$\times$\\
score-variance share (11\% of dims) & \scoreVarShare & (branch is positional)\\
top-1 retention, branch only & \branchTopOne & n/a\textsuperscript{$\ast$}\\
top-1 retention, content only & \contentTopOne & n/a\textsuperscript{$\ast$}\\
normalization & \multicolumn{2}{l}{content is RMSNormed; the branch is the row's only magnitude path}\\
\bottomrule
\end{tabular}
\caption{The branch is the trained salience channel. Static rows from weights alone;
behavioral rows from 512 real queries $\times$ 32 heads.
$^{\ast}$Retention is a norm-type ranking and rotations are isometries, so
the statistic is insensitive to RoPE by construction (measured: 0.99 on DSV2
as well) --- it cannot discriminate the arms. The discriminating measurement
is the end-to-end dual-arm selector (\ropeVKxThirtyTwo{} on DSV2).}
\label{tab:anatomy}
\end{table}

\begin{table}[H]
\centering
\caption{Tier-1 ablation. Left: dual-arm acceptance, frozen before either arm ran: work on NoPE
\emph{and} fail on RoPE (DeepSeek-V2-Lite, matched layers, $32\times$). Right:
branch-width ablation --- post-hoc PCA truncation before $\sigma$; monotone
degradation below 32 dims.}
\label{tab:rope}
\small
\begin{tabular}{lcc}
\toprule
 & NoPE & RoPE\\
\midrule
VestigeKV & \vkEightKxThirtyTwo & \ropeVKxThirtyTwo\\
full-row eviction & \fullEightKxThirtyTwo & \ropeEvictxThirtyTwo\\
top-1 union & \unionKimi & \unionDsv\\
\bottomrule
\end{tabular}\hspace{1.4em}
\begin{tabular}{lccccc}
\toprule
branch dims & 4 & 8 & 16 & 32 & 64\\
\midrule
@ $32\times$  & \ablFourxThirtyTwo & \ablEightxThirtyTwo & \ablSixteenxThirtyTwo & \ablThirtyTwoxThirtyTwo & \ablSixtyFourxThirtyTwo\\
@ $128\times$ & \ablFourxOTE & \ablEightxOTE & \ablSixteenxOTE & \ablThirtyTwoxOTE & \ablSixtyFourxOTE\\
\bottomrule
\end{tabular}
\end{table}

\begin{table}[H]
\centering\small
\caption{Left: end-to-end recovery ($L{=}8192$, 24 trials) --- the tier restores
what eviction loses, admitting only \ttMaskExtra{} extra rows per query. Right: the
index-rank knob --- a larger sketch cuts trigger rate and fetch volume while recall
\emph{rises} (tighter bounds rank better). Rows/query is a document-perplexity
figure; retrieval-heavy steps legitimately fetch more (5--37 rows per layer at
$r{=}192$ on needle contexts, recovery still 1.00).}
\label{tab:tier}
\begin{tabular}{lcc}
\toprule
 & $32\times$ & $128\times$\\
\midrule
eviction only & \vkEightKxThirtyTwo & \vkEightKxOTE\\
$+$ recall tier & \ttxThirtyTwo & \ttxOTE\\
\bottomrule
\end{tabular}\hspace{1.6em}
\begin{tabular}{lcccc}
\toprule
sketch rank $r$ & 64 & 128 & 192 & 256\\
\midrule
fire rate & \rsSixtyFourFire & \rsOneTwoEightFire & \rsOneNineTwoFire & \rsTwoFiveSixFire\\
rows/query & \rsSixtyFourRows & \rsOneTwoEightRows & \rsOneNineTwoRows & \rsTwoFiveSixRows\\
hard recall & \rsSixtyFourRec & \rsOneTwoEightRec & \rsOneNineTwoRec & \rsTwoFiveSixRec\\
\bottomrule
\end{tabular}
\end{table}

\begin{table}[H]
\centering\small
\caption{General language-modeling cost ($\Delta$CE, nats/token, held-out; the
top row is the deployed two-tier form, the lower rows its ablations
continuation, 6 docs). The standard configuration is effectively lossless
(${\sim}0.0006$ bits/byte at $32\times$ for ${\sim}4$-byte tokens). The
recent-only floor matches on CE while scoring $+15$ nats on needle --- continuation
loss alone cannot arbitrate retrieval, which is why needle is the primary metric.}
\label{tab:bpb}
\begin{tabular}{lccc}
\toprule
$\Delta$CE & $8\times$ & $32\times$ & $128\times$\\
\midrule
eviction $+$ recall tier (standard) & \bpbTTxEight & \bpbTTxThirtyTwo & \bpbTTxOTE\\
eviction only & \bpbDigxEight & \bpbDigxThirtyTwo & \bpbDigxOTE\\
recent-only floor & \bpbRecxEight & \bpbRecxThirtyTwo & \bpbRecxOTE\\
\bottomrule
\end{tabular}
\end{table}

\begin{table}[H]
\centering\small
\begin{tabular}{ll>{\raggedright\arraybackslash}p{5.7cm}}
\toprule
parameter & recommended & basis\\
\midrule
ratio $\rho$ & $1/8$ lossless; $1/32$ default & Table~\ref{tab:main}\\
bandwidth $\kappa$ & 16 ($L\le16$k); 64 ($L\ge32$k) & tracks context length\\
sinks / recent & first 4 / one block & sink dominance of attention mass\\
sketch rank $r$ & 64 & minimizes per-step scan bytes (reads
 26\% of full at $r{=}64$ vs 47\% at $r{=}192$; the 16-bit index halves both);
 larger $r$ buys only a lower fire rate (\rsOneNineTwoFire{} at $r{=}192$,
 recall \rsOneNineTwoRec) that the fused pipeline no longer needs\\
fetch & uncapped (default) & the fired set is the certificate's decision,
 not a quota; $\mathrm{topj}{=}16$ is an optional worst-case-latency
 optimization (recall 0.93--0.98 across data families at the cap)\\
trigger $z$ & conformal quantile at $\tau_s$, from $\tau{=}0.90$ & closed form; in-context; no external corpus\\
entropy gate & auto-off if cal.\ fire rate $>60\%$ & self-deciding; saves the scan only where it filters\\
tier-2 placement & GPU-resident & fetch free; per-step reads
 $1152\to{\sim}294$ bytes/token (16-bit index)\\
\bottomrule
\end{tabular}
\caption{Empirical defaults. Every value traces to a pre-registered measurement.}
\label{tab:config}
\end{table}

\section{Tier-1 ablation detail}\label{app:tier1}
The deployed form always runs both tiers; the tables below switch the recall
tier \emph{off} to isolate the selector. They support one sentence of the
main text --- without the recall, the partition is just lossy eviction ---
and are not final quality.

\begin{table}[H]
\centering
\caption{Tier-1 ablation (recall tier off; not the deployed form). Rates at
8k and 32k pool seeds 11$+$12 under the frozen rule of the pre-registration
archive ($n{=}24$--$36$); 65k is single-seed ($n{=}12$).
Branch-only selection (VestigeKV, reads \readCost{} of each row) vs.\
full-row selection at matched budget.}
\label{tab:main}
\footnotesize
\begin{tabular}{lccccccccc}
\toprule
 & \multicolumn{3}{c}{$L{=}8192$} & \multicolumn{3}{c}{$L{=}32768$} & \multicolumn{3}{c}{$L{=}65536$}\\
 & $32\times$ & $64\times$ & $128\times$ & $32\times$ & $64\times$ & $128\times$ & $32\times$ & $64\times$ & $128\times$\\
\midrule
VestigeKV & \vkEightKxThirtyTwo & \vkEightKxSixtyFour & \vkEightKxOTE & \vkThirtyTwoKxThirtyTwo & \vkThirtyTwoKxSixtyFour & \vkThirtyTwoKxOTE & \vkSixtyFiveKxThirtyTwo & \vkSixtyFiveKxSixtyFour & \vkSixtyFiveKxOTE\\
Full-row selection & \fullEightKxThirtyTwo & \fullEightKxSixtyFour & \fullEightKxOTE & \fullThirtyTwoKxThirtyTwo & \fullThirtyTwoKxSixtyFour & \fullThirtyTwoKxOTE & \fullSixtyFiveKxThirtyTwo & \fullSixtyFiveKxSixtyFour & \fullSixtyFiveKxOTE\\
\bottomrule
\end{tabular}
\end{table}

\begin{table}[H]
\centering
\caption{Tier-1 ablation (recall tier off). Needle intact rate when compression precedes the query ($L{=}8192$, 24
trials; $L{=}32768$, 12 trials). The recent-window variant rules out a budgetary
explanation for the collapse (Section~\ref{sec:problem}). These methods remain strong in their design setting
(query present at compression); this is the other setting.}
\label{tab:baselines}
\begin{tabular}{lccc}
\toprule
 & $8\times$ & $32\times$ & $128\times$\\
\midrule
H2O \citep{zhang2023h2o} & \blHtwoOxEight & \blHtwoOxThirtyTwo & \blHtwoOxOTE\\
SnapKV \citep{li2024snapkv} & \blSnapKVxEight & \blSnapKVxThirtyTwo & \blSnapKVxOTE\\
H2O $+$ recent (half budget) & \blHRecxEight & \blHRecxThirtyTwo & \blHRecxOTE\\
StreamingLLM \citep{xiao2024streamingllm} & 0.00 & 0.00 & 0.00\\
VestigeKV & \vkEightKxEight & \vkEightKxThirtyTwo & \vkEightKxOTE\\
\midrule
H2O / SnapKV, $L{=}32768$ & \blHtwoOLxLEight\,/\,\blSnapKVLxLEight & \blHtwoOLxLThirtyTwo\,/\,\blSnapKVLxLThirtyTwo & \blHtwoOLxLOTE\,/\,\blSnapKVLxLOTE\\
VestigeKV, $L{=}32768$ & \vkThirtyTwoKxEight & \vkThirtyTwoKxThirtyTwo & \vkThirtyTwoKxOTE\\
\bottomrule
\end{tabular}
\end{table}

\section{Measurement integrity: full protocol}\label{app:integrity}
Every check below was driven with a known-bad input and shown to refuse it before
being trusted; every threshold in this paper was frozen, dated, before its data.
All reported rates are exact fractions over the stated trial counts;
Wilson 95\% intervals for every primary rate are in
Appendix~\ref{app:wilson}.

Table~\ref{tab:gates} (Appendix~\ref{app:wilson}) lists each gate with the
known-bad input it was proven to catch.

\paragraph{Executable row-set invariant.} The serving port extends the
same discipline to the deployment: a per-decode-step assertion
(\texttt{SGLANG\_DEBUG\_VESTIGEKV\_ROWS}) checks, on every pipeline rank and
every layer, that the attended row set is the intended one --- the kept
table exists, no active slot is empty, the baseline arm attends the full
prefix, the compressed arm attends well below it. Each failure mode is
unit-tested to fire on the broken form before being trusted, and the final
benchmark campaign ran a full proof round with the assertion enabled on
both ranks and both arms across all batch sizes: zero violations.

\paragraph{Serving-stack performance methodology.} All wall-clock numbers
are produced by \texttt{sglang}'s own benchmark module
(\texttt{sglang.benchmark.serving} with per-token ITLs saved raw), with no
benchmarking code of our own; the orchestration only launches the server
and switches the arm. Both arms run from one source tree with identical
CUDA-graphed code paths, each arm's curve is collected within a single
launch (fixed launch-state terms cancel in the $\Delta$-vs-4k reading),
above ${\sim}$224k context the server-side decode log is the per-token
authority (client-side inter-token stamps degrade under stream batching,
a measured artifact), and every launch log prints the checkpoint's
content-level fingerprint so all arms provably load identical weights.

\section{Scope boundaries: full table}\label{app:limits}
Every scope boundary of this study, with its measured status:

\begin{center}\small
\begin{tabular}{l>{\raggedright\arraybackslash}p{8.4cm}}
\toprule
limitation & status\\
\midrule
one measured model & all measurements are on Kimi Linear 48B (the
\emph{Base} checkpoint; post-training survival of the signal is untested).
For Kimi K3's Gated-MLA the machinery extends by the gate-commutation
argument; the branch's salience \emph{content} there is an untested
empirical claim\\
one NoPE-MLA family & Kimi Linear 48B is the only released NoPE-MLA checkpoint;
selector \emph{ordering} replicates on a NoPE GQA hybrid (Granite-4.0-H), eviction
\emph{depth} does not --- depth claims are MLA-conditional\\
mechanism of the depth & open; three pre-registered candidate explanations were
refuted by our own runs, recorded in the pre-registration archive\\
row reachability & existence is constructive (the partition deletes nothing);
per-step consultation is probabilistic (recall 0.9--1.0, 1.00 end to end;
the fired set itself is deterministic by construction given the query).
Pure deletion loses rows irrecoverably ($+15$ nats of NLL on the needle answer)\\
tasks, length & needle probes to 256k, held-out CE at 8k, 64-shot QA
(${\sim}$9k prompts, $n{=}800$) and MAUVE on the serving path; serving
latency measured to 512k. Standard long-context suites are future work\\
\bottomrule
\end{tabular}
\end{center}

\section{Interval estimates for the primary rates}\label{app:wilson}
Needle intact rates are exact fractions over small trial counts; Wilson 95\%
intervals make the resolution explicit. Neighboring ratios at one length are
often not separated at these $n$. The separations the paper leans on hold:
tier-1 vs.\ the recent-window arm (0 successes in every cell) and tier-2
recovery vs.\ its eviction arm. The flatness of $1/32$ across lengths is a
consistency observation (identical counts at all three lengths), not a
separation claim.

\begin{table}[H]
\centering\small
\caption{Wilson 95\% intervals, count/trials [lower, upper]. Tier-1 ablation
rows above the rule; deployment-loop rows below.}
\begin{tabular}{lccc}
\toprule
 & $32\times$ & $64\times$ & $128\times$\\
\midrule
$L{=}8192$ & \wciEightKxThirtyTwo & \wciEightKxSixtyFour & \wciEightKxOTE\\
$L{=}32768$ & \wciThirtyTwoKxThirtyTwo & \wciThirtyTwoKxSixtyFour & \wciThirtyTwoKxOTE\\
$L{=}65536$ & \wciSixtyFiveKxThirtyTwo & \wciSixtyFiveKxSixtyFour & \wciSixtyFiveKxOTE\\
\midrule
recall tier, $128\times$, 8k & \multicolumn{3}{c}{\wciTT}\\
deployment loop, tier-1, $32\times$ & \multicolumn{3}{c}{\wciMFour}\\
deployment tier-2, $512\times$@8k, $256\times$@32k & \multicolumn{3}{c}{\wciMFiveTier, \wciMFiveTK}\\
\bottomrule
\end{tabular}
\end{table}

\begin{table}[H]
\centering
\caption{Gates and process.}\label{tab:gates}
\small
\begin{tabular}{ll}
\toprule
gate & known-bad input it was proven to catch\\
\midrule
identity operator ($\Delta$CE $=0$) & \texttt{eager}/\texttt{sdpa} kernel mismatch ($1.4\times10^{-1}$)\\
uncompressed-baseline gate & trials the base model cannot retrieve\\
hook fire counts & silent non-capture (verdicts about nothing)\\
leakage floor (cache destroyed $\Rightarrow0.00$) & uncompressed layers recovering the needle\\
pre-registration, \nVerdicts{} verdicts & three retracted mechanisms; a rejected learned selector\\
\bottomrule
\end{tabular}
\end{table}

\section{Deployment specification}\label{app:spec}
The serving schedule end to end, for an engineer holding stock weights.
Each item is tagged: [M] = measured in this paper's record; [D] = design,
licensed by the stated result but not itself wall-clock-measured.

\paragraph{Lifecycle of a request.}
\begin{enumerate}
\item \textbf{Below the activation threshold} $L^{*}$: the partition is off;
  the attended set is everything, no index, no compaction --- byte-identical
  to the stock MLA path at zero added cost. [D; the switch-on is
  retroactive-tax-free by the streaming equalities of
  Section~\ref{sec:exp} [M]]
\item \textbf{Block close} (every 4096 rows past $L^{*}$): $\sigma$ from the
  64-dim branch, global top-$m$ rebalance, sinks and open tail always
  attended. [M: trial-exact vs one-shot]
\item \textbf{Compression event} (prefill end): tier-2 index build ---
  64-dim exact summand, rank-$r$ sketch from the context's own queries,
  per-row certificate; $z$ and the entropy gate self-calibrate from the
  prefix; the gate disables itself where it cannot calibrate. [M]
\item \textbf{Per decode step}: attended-tier attention; index scan; fetch
  top-$j{=}16$ fired rows into this step. [M: recovery 1.00 at $128\times$
  and $512\times$/8k, $256\times$/32k in the serving loop]
\item \textbf{Optional per-layer cascade}: stage-1 sidecar shortlist enabled
  per layer iff calibration hard-target recall $\ge 0.90$
  (holdout-verified); scan $\times 0.65$ (8k) / $\times 0.79$ (32k); the
  uniform variant is refuted and must not ship. [M]
\end{enumerate}

\paragraph{Placement and knobs.} Ratio target $1/32$, floor $1/32$ for
constant-$m$ schedules (the only ratio measured flat in length) [M];
sketch rank $r{=}64$: tier-2 is GPU-resident (fetches free), so $r$
minimizes scan bytes [M]; tier-2 is mandatory --- the partition deletes
nothing [M]. Cache rows are never quantized (the controlled-variable rule);
the scan's index metadata is stored 16-bit with fp32-ieee arithmetic and
post-rounding calibration [M].

\paragraph{Asynchrony contract.} All block-close work ($\sigma$, rebalance,
compaction, pre-threshold precomputation) may run off the critical path:
inputs are immutable rows, early/late computation is equivalent by
Corollary~\ref{cor:stationary}, and a lagging update only enlarges the
attended set by one block --- the conservative direction. Sync points: a
per-block write event, a barrier before the index build, an atomic swap of
the attended set. [D, licensed by [M] equalities]

\paragraph{Kernel mapping.} No new kernels: kept rows compact once into
pages (final by Corollary~\ref{cor:stationary}) and run on stock paged
attention; the gate consumes the LSE that kernel already emits; the index
scan is one GEMV per layer; fetched rows merge as one extra split-KV
partition. [D; the LSE-gate substitution itself is measured]

\section{The dividend inside prior algorithms}\label{app:dividends}
\textbf{Claim up front: NoPE confers a family of algorithmic dividends, not
one.} Two structural facts --- the score is a single fixed bilinear form,
and rankings commute with time (Corollary~\ref{cor:stationary}) --- turn
into distinct, unclaimed guarantees inside one published cache-management
algorithm after another once transplanted to NoPE-MLA: statistics become
stationary, voting horizons become unbounded, screening bounds become
admissible, projection objectives become clean, merging acquires an
exactness certificate, and reuse becomes zero-copy. The main text derives
two instances (the selector, Section~\ref{sec:math}; page digests,
Section~\ref{sec:arkvale}); this appendix enumerates the rest. Each is a
derivation from the setup of Section~\ref{sec:math}; none is a new
measurement, and we say so where the measured record limits the claim.

\paragraph{Accumulated-attention eviction (H2O, TOVA).}
These methods score row $u$ by attention mass observed under past queries,
$\hat{s}(u)=\sum_t \mathrm{softmax}\text{-weight}(q_t,u)$
\citep{zhang2023h2o,oren2024tova}. What the statistic estimates:
\begin{equation}
\mathbb{E}\big[\hat{s}(u)\big]=
\begin{cases}
\,n\,\mathbb{E}_{q}\!\left[w\big(\lambda\langle Q,\tilde{C}_u\rangle\big)\right]
 & \text{NoPE: no } t \text{ enters --- one number per row,}\\[2pt]
\,\sum_{t} \mathbb{E}_{q}\!\left[w\big(\lambda\, q^{\top}R_{t-u}W_{K}\tilde{C}_u\big)\right]
 & \text{RoPE: a function of the age profile } \{t-u\}.
\end{cases}
\end{equation}
Under NoPE the estimand is time-invariant, so the estimate transfers to every
future query from the same distribution. Under RoPE a future query at
$t'=t+L_{\mathrm{gen}}$ evaluates the integrand at a rotation the statistic
never sampled; by the band expansion of Corollary~\ref{cor:stationary} the
mismatch oscillates per frequency, so the heavy-hitter statistic has a shelf
life set by the band sum. NoPE removes the shelf life.

\paragraph{Observation-window voting (SnapKV).}
SnapKV votes with the last window's queries immediately before generation
\citep{li2024snapkv}. The score a vote certifies and the score generation
consumes differ, for the same $(q,u)$, by exactly
\begin{equation}\label{eq:votedrift}
s_{t'}(q,u)-s_{t}(q,u)=\lambda\, q^{\top}\big(R_{t'-u}-R_{t-u}\big)W_{K}\tilde{C}_u,
\qquad t'=t+L_{\mathrm{gen}},
\end{equation}
which vanishes identically under NoPE ($R\equiv I$) and oscillates per
frequency band under RoPE. The vote's validity horizon is therefore unbounded
under NoPE and equal to the generation length under RoPE.

\paragraph{Attention sinks (StreamingLLM).}
Equation~\ref{eq:votedrift} applied to a fixed sink row $u$: under NoPE
the sink's score against any query direction is length-invariant, so the
mechanism of \citet{xiao2024streamingllm} is exactly stable; under RoPE the
same difference term rides the lowest frequency bands and drifts with
distance.

\paragraph{Page-bound pruning (Quest).}
Quest skips pages whose per-page score bound falls below the running best
\citep{tang2024quest}. Under NoPE the score is linear in the row for a
fixed form, so interval arithmetic over a per-page box
$B(P)=[\ell,h]\supseteq\{\tilde{C}_u\}_{u\in P}$ gives a true bound,
\begin{equation}
\max_{u\in P}\, \lambda\langle Q,\tilde{C}_u\rangle
\;\le\; \lambda\!\sum_{d}\max\big(Q_d \ell_d,\; Q_d h_d\big),
\end{equation}
and pruning by it is \emph{admissible}: no page containing the argmax is
skipped beyond box slack. Measured caution: admissible is not automatically
useful --- on this checkpoint, boxes over the 256-dim index at page size 64
prune under $0.2\%$ of pages (per-dim maxima over 64 rows exceed every
actual row's score in high dimension); tight page digests are the open
problem, the guarantee is the opening. Under RoPE the left side is
$\max_{u\in P}\lambda\, q^{\top}R_{\Delta_u}W_K\tilde{C}_u$ with
$\Delta_u$ varying \emph{inside} the page, so an honest bound must also
maximize over the rotation orbit --- the radius inflation of
Eq.~\ref{eq:radius} --- and tightness is lost with page span. Combined with
the measured layer-uniformity of NoPE page digests
(Section~\ref{sec:arkvale}), this is the theoretical opening for a
hierarchical recall index with a per-step scan sublinear in context length;
we have not built or measured one.

\paragraph{Trainable projections (MatryoshkaKV).}
For an orthogonal projector $\Pi$ on the content latent, the NoPE score
error is exactly $|\lambda\langle Q,(I-\Pi)\tilde{C}_u\rangle|$ --- a
clean, query-uniform objective \citep{lin2025matryoshkakv}. Under RoPE the
same objective is position-dependent unless $\Pi$ near-commutes with the
whole rotation family,
\begin{equation}
\sup_{\Delta}\big\lVert R_{\Delta}^{\top}\Pi R_{\Delta}-\Pi\big\rVert
\;\approx\;0,
\end{equation}
which ties the projector to the frequency pairing and shrinks the feasible
set to band-aligned subspaces.
The measured record bounds the claim: on the frozen Kimi Linear checkpoint
the read-out is full-rank (rank $512$), so this dividend accrues to future
training (Section~\ref{sec:outlook}), not to the deployed weights.

\paragraph{Lossless merging (KeepKV).}
KeepKV merges cache entries with the ambition of losslessness
\citep{tian2025keepkv}. Proposition~\ref{prop:impossible} is the exactness
certificate that ambition needs, and it is available only under NoPE-MLA:
identical rows merge exactly with an $\ln m$ bias, the RMSNorm homogeneity
class enlarges the merge set, and under RoPE the merge class is empty.
The measured record again bounds the claim: on real corpora the tolerant
merge class is empty at $5\%$ tolerance, so the certificate currently has
no application domain on this checkpoint.

\paragraph{Position-independent reuse (Irminsul, Kamera).}
Cross-request cache reuse \citep{ma2026irminsul,ma2026kamera} is the
exchangeability lemma applied across sequences: a NoPE-MLA row is valid at
any position in any context, so reuse is zero-copy. RoPE reuse pays a
re-rotation of every moved row. These systems exploit the property; the
lemma names it.

\end{document}